%% file: main.tex
\documentclass[letterpaper, 10 pt, conference]{ieeeconf}  % Comment this line out if you need a4paper

\IEEEoverridecommandlockouts
\usepackage{graphics} % for pdf, bitmapped graphics files
\usepackage{epsfig} % for postscript graphics files
\usepackage{mathptmx} % assumes new font selection scheme installed
\usepackage{times} % assumes new font selection scheme installed
\usepackage{amsmath} % assumes amsmath package installed
\usepackage{amssymb}  % assumes amsmath package installed

\usepackage{amsthm}
\usepackage[ruled,vlined,linesnumbered]{algorithm2e} % for algorithm environment
\usepackage[font=small, skip=4pt]{caption} % Remove giant spaces between figures and captions
\usepackage{xcolor} % for colors
\usepackage{float} % for finer control over figure placement
\usepackage{makecell} % for multiline cells in tables

\DeclareMathAlphabet{\mathcal}{OMS}{cmsy}{m}{n}

\newcommand{\statespace}{X}
\newcommand{\actionspace}{U}
\newcommand{\vx}{\mathbf{x}}
\newcommand{\vu}{\mathbf{u}}

\newcommand{\vq}{\mathbf{q}}
\newcommand{\vp}{\mathbf{p}}
\newcommand{\vz}{\mathbf{z}}
\newcommand{\vd}{\mathbf{d}}

\usepackage{hyperref}
\hypersetup{
    colorlinks=true,
    linkcolor=blue,
    filecolor=magenta,      
    urlcolor=cyan,
    }

\input{glossary}

\makeatletter
\newcommand{\nosemic}{\renewcommand{\@endalgocfline}{\relax}}% Drop semi-colon ;
\newcommand{\dosemic}{\renewcommand{\@endalgocfline}{\algocf@endline}}% Reinstate semi-colon ;
\let\oldnl\nl% Store \nl in \oldnl
\newcommand{\nonl}{\renewcommand{\nl}{\let\nl\oldnl}}% Remove line number for one line
\newcommand{\pluseq}{\mathrel{+}=}

\DeclareMathOperator*{\argmax}{arg\,max}

\newtheorem{theorem}{Theorem}
\newtheorem{assumption}{Assumption}
\newtheorem{remark}{Remark}
\newtheorem{lemma}{Lemma}

\title{\LARGE \bf
Model Predictive Trees: Sample-Efficient Receding Horizon \\ Planning with Reusable Tree Search
}

\author{John Lathrop, Benjamin Rivière, Jedidiah Alindogan, Soon-Jo Chung%
\thanks{\hspace{-1.0em}This work was supported by NSF Grant 2139433 and in part by Supernal.}%
\thanks{All authors are with the California Institute of Technology.
        1200 E.~California Blvd., Pasadena CA 91106, USA
        {\tt\small \{jlathrop,briviere,jedi,sjchung\}@caltech.edu}}}%

\begin{document}

\maketitle
\thispagestyle{empty}
\pagestyle{empty}

%%%%%%%%%%%%%%%%%%%%%%%%%%%%%%%%%%%%%%%%%%%%%%%%%%%%%%%%%%%%%

\begin{abstract}

% We present a new sampling-based planner, \ac{mpt}, a real-time receding horizon tree search algorithm that shows sample-efficient performance in sparse reward environments.
% Our technology reuses subtree information from previous solve iterations to accelerate the search process, making full use of previously simulated trajectories from the selected subtree.
% In contrast to deterministic or sampling-based solvers that hotstart with the previous optimal solution, reusing entire subtrees allows the search to quickly discard low-valued regions of the search space and concentrate on refining high-valued trajectories.
% We analyze the closed-loop stability and robustness of our proposed algorithm when combined with contraction-theoretic control.
% \todo{emphasize the feedback between the controller and planner}
% In numerical studies, our algorithm outperforms state-of-the-art sampling-based cross-entropy methods.
% We demonstrate our planner on an autonomous vehicle testbed performing a nonprehensile manipulation task: pushing a target object through an obstacle field.
% Code associated with this work will be made available at \textbf{https://github.com/jplathrop/mpt}.

% We present a new sampling-based planner, \ac{mpt}, a real-time receding horizon tree search algorithm that shows sample-efficient performance in sparse reward environments.
We present Model Predictive Trees (MPT), a receding horizon tree search algorithm that improves its performance by reusing information efficiently.
% In a receding horizon context, it is possible to improve solver's performance at the current time by reusing information from previous iterations.
% Whereas existing solvers reuse only the highest-quality trajectory from the previous iteration as a ``hotstart'', our method reuses the entire subtree, enabling the search to quickly ignore low-valued regions of the search space and instead concentrate on refining high-valued trajectories.
Whereas existing solvers reuse only the highest-quality trajectory from the previous iteration as a ``hotstart'', our method reuses the entire optimal subtree, enabling the search to be simultaneously guided away from the low-quality areas and towards the high-quality areas. 
We characterize the restrictions on tree reuse by analyzing the induced tracking error under time-varying dynamics, revealing a tradeoff between the search depth and the timescale of the changing dynamics. 
% We characterize the limit of tree reuse by analyzing the induced tracking error in time-varying dynamics; at this limit, the tree can no longer be reused and the solver must start from scratch. 
In numerical studies, our algorithm outperforms state-of-the-art sampling-based cross-entropy methods with hotstarting.
We demonstrate our planner on an autonomous vehicle testbed performing a nonprehensile manipulation task: pushing a target object through an obstacle field.
Code associated with this work will be made available at \textbf{https://github.com/jplathrop/mpt}.

\end{abstract}

%%%%%%%%%%%%%%%%%%%%%%%%%%%%%%%%%%%%%%%%%%%%%%%%%%%%%%%%%%%%%

\section{Introduction}

Gradient-free optimization techniques are an attractive framework for decision making and motion planning on robotic systems where high-fidelity models may not be differentiable and descent algorithms can get caught in local minima.
% As a motivating example, we consider nonprehensile manipulation, which is a challenging task due to its hybrid dynamics and sparse reward structure.
% Both of these challenges are difficult to address with traditional gradient-based optimization methods.
As a motivating example, we consider nonprehensile manipulation, a setting where a robot uses pushing, pulling, or other means of manipulation without grasping the target object with appendages. 
This task is challenging for conventional gradient-based optimization because of its hybrid dynamics and sparse reward structure.

\ac{uct}, a variant of \ac{mcts}, is a powerful gradient-free technique that strategically explores the space of possible future trajectories, with guaranteed convergence to the optimal trajectory as its runtime increases~\cite{kocsis2006improved}.
Although \ac{uct} is widely applicable to a large class of decision-making problems, its performance is dependent on its computational resources and accumulating enough samples to make an accurate value estimate.
% Making prudent use of past simulations can improve the quality of search and reduce sample complexity, particularly when the simulation of dynamics is computationally expensive.
With this context, there is clear benefit in reusing past simulations to improve the quality of search and reduce sample complexity. 
In this work, we propose an efficient sub-tree recycling procedure and characterize the conditions under changing dynamics where tree recycling is not possible and instead must be recomputed from scratch. 

\subsubsection*{Contributions}
We present \ac{mpt}, a receding horizon tree-based planner that reuses past subtree information, reducing the cost of searching while greatly boosting the quality of solutions.
Saving even low-quality trajectories benefits the search by shifting computational effort from re-generating poor solutions to refining high-quality solutions.
We demonstrate our proposed method on a nonprehensile manipulation task performed by an autonomous car in simulation and performing in real-time on onboard hardware. 
We furthermore provide theoretical guarantees on stability and robustness in the face of changing dynamics.
Our method automatically discovers high-level behavior, strategically making and breaking contact with a target object while maneuvering around obstacles to push the target to the goal.
% Our method is computationally efficient to run in real-time onboard our test platform and offers a significant improvement over state-of-the-art sampling-based methods, while providing theoretical guarantees on stability and robustness to changing dynamics.

\begin{figure}[t]
    \centering
    \includegraphics[width=1.0\linewidth]{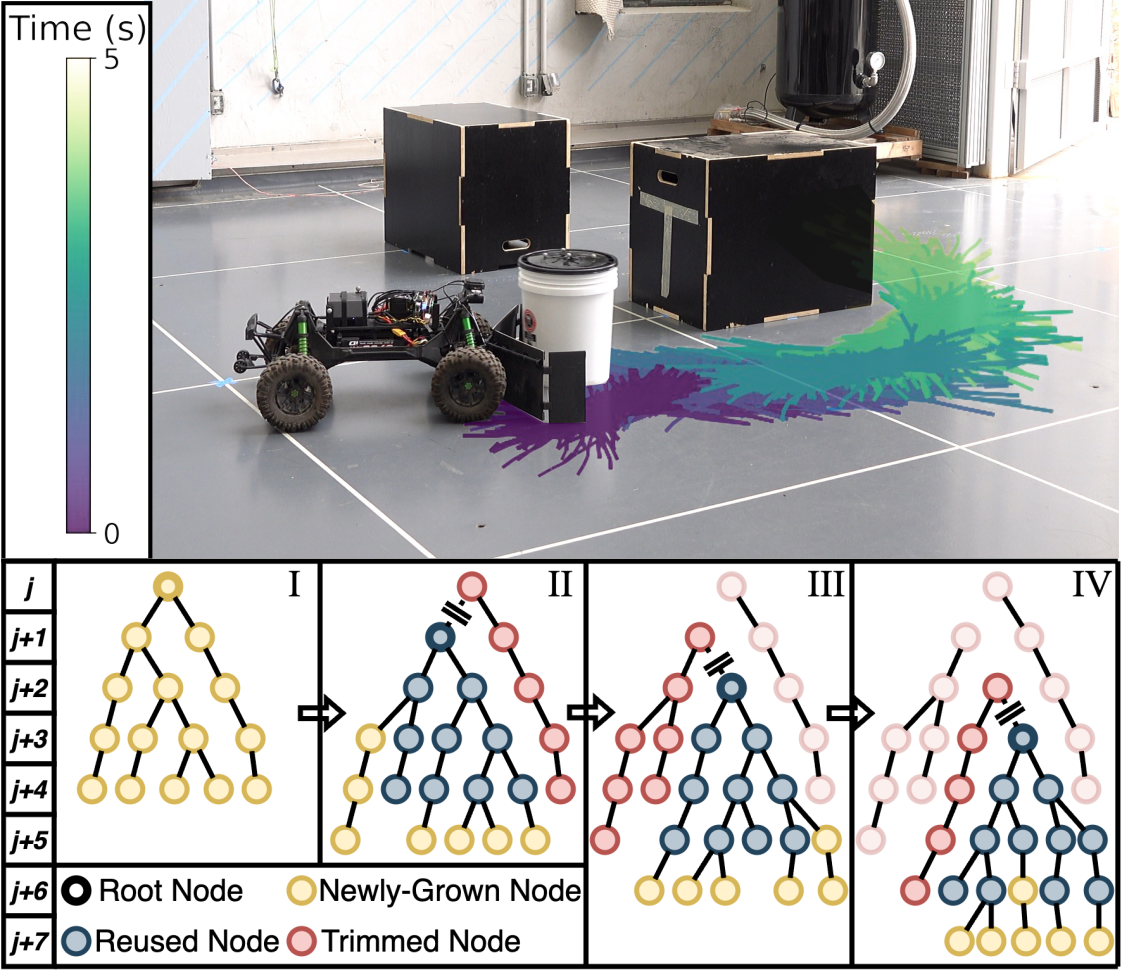}
    \caption{
    \textbf{Top:} Five seconds of real-time generated trees in our hardware experiment, in which the autonomous vehicle testbed pushes a target to a goal region behind an obstacle. On average, 2100 simulated trajectories are grown every 0.2 s.
    The trees are colored by the time they were grown.
    \textbf{Bottom:} The proposed tree growth algorithm visualized over four time steps. At each iteration (I - IV), new nodes and branches are added to the tree, and the best first-level child is selected as the next root.
    In subsequent iterations, older parts of the tree are discarded and new nodes are added.}
    \label{fig:tree_reuse}
\end{figure}

\subsection{Related Work}

Our related work spans sampling-based planning, strategies for reusing data from previous iterations, control-aware planning, and nonprehensile manipulation.  

\subsubsection{Sampling-Based Planners}

Sampling-based tree search is common in robotics, with foundational techniques including \ac{rrt})~\cite{lavalle1998rapidly}, RRT* and variations~\cite{karaman2011sampling, otte2016rrtx}, and \ac{prm}~\cite{kavraki1996probabilistic} as well-tested and theoretically-grounded algorithms.
These methods use a local planner to connect states and search for a goal region.
The stable sparse RRT variation~\cite{li2015sparse} relaxes the need for a local planner.
% , but is still limited by a goal region. 

In contrast, our method uses tree search algorithms investigated in theoretical computer science, in particular the \ac{uct} algorithm~\cite{kocsis2006improved}.
\ac{uct} and variants have been applied in a variety of problem settings, including robotic task planning~\cite{labbe2020monte}, motion planning~\cite{riviere2021neural}, and active sensing~\cite{ragan2023bayesian}.
The advantage of using \ac{uct} is that our method can plan and execute through hybrid contact dynamics, where the sparse reward and non-smooth contacts make local planning and the restriction to a goal region-objective difficult, therefore limiting the applicability of RRT-based algorithms.

\subsubsection{Data Reuse in Receding Horizon Planning}

Real-time planning is commonly implemented with a receding horizon approach, also known as model-predictive control, where the solver iteratively computes and follows a finite-horizon trajectory until the process is terminated. 
In this setting, there are various strategies to reuse information from the previous solver iterations in the current iteration.
% , but their ability to reuse information from previous iterations is limited. 
\ac{cem}~\cite{kobilarov2012cross} and \ac{mppi}~\cite{williams2016aggressive} are sampling-based methods that initialize the sampling distribution of the next iteration with the mean of the optimal solution from the previous iteration. 
This approach, sometimes called "hotstarting", averages many simulated trajectories into a summary statistic and throws away a large amount of information.

Deterministic nonlinear model predictive control uses a similar technique to provide an initial guess for the optimization solver~\cite{diehl2002real} with the optimal solution to the previous iteration. 
Other works have examined tree reuse and tree correction in a control task~\cite{schneider2016receding} and with changing obstacles~\cite{otte2016rrtx}. 
% Tree reuse is also related to multi-query path planning~\cite{kavraki1996probabilistic}, although our work  focuses on the real-time receding-horizon context.

In contrast to other tree reuse strategies, our method saves the entire selected subtree from the previous iteration, leveraging much more information to refine its search.
Our experimental results demonstrate that hotstarting with a subtree provides a larger improvement than hotstarting with the optimal solution, validating our intuition about information reuse.
In addition, our reuse approach does not require iteration over the entire tree, a time-consuming requirement that limits real-time deployment.

\subsubsection{Nonprehensile Manipulation}
% For completeness, we include some related works that demonstrate nonprehensile manipulation.
% Our focus is not on this particular problem, though through it we demonstrate the ability of our algorithm to plan through contact.
% Finally, we include some related works that demonstrate nonprehensile manipulation.
Although our approach can be applied to a wide class of problems, we focus on nonprehensile manipulation as a motivating example because of its inherent complexity for conventional techniques due to the need to plan through contact. 
% In particular, we demonstrate the ability of our algorithm to plan through contact.
% A thorough review of the challenges of nonprehensile manipulation are available in other works~\cite{pang2023global}.
In previous work~\cite{pang2023global}, the authors demonstrate nonprehensile manipulation with a smoothed contact model.
In contrast, we do not make a smooth approximation of the contact model, but plan directly in the underlying sparse reward and hybrid dynamical landscape. 
Whereas other methods use constraints to maintain contact~\cite{bertoncelli2020linear, zhang2020motion, selvaggio2023non}, our method does not require contact to be maintained: \ac{mpt} demonstrates high-level behavior, such as backing-up and re-positioning, where the vehicle maneuvers around the object to execute a better push.
% These methods do explicitly model the frictional forces involved, while we only consider the rigid-body contact modeling.
% Whereas existing methods that enforce contact as a constraint explicitly model friction, we consider only the rigid-body contact modeling, and rely on the planner's interaction with the controller to handle the modeling error. 

Finally, deep reinforcement learning has been applied to nonprehensile manipulation~\cite{yuan2018rearrangement} in simulation. 
% However, this family of methods requires a large amount of offline training data, has poorly understood optimality convergence, and its relevance to real-world applications has yet to be demonstrated.
However, this family of methods requires a large amount of offline training data and lacks theoretical guarantees of optimality and stability.

\section{Model Predictive Trees}\label{sec:algorithm}

\subsection{Problem Setting}

We consider the problem of making decisions over an infinite horizon.
For a compact state space $\statespace \subset \mathbb{R}^n$, compact action space $\actionspace \subset \mathbb{R}^m$, consider a discrete-time control system characterized by known nominal dynamics $F_\text{nom}: \mathbb{R}^n \times \mathbb{R}^m \rightarrow \mathbb{R}^n$
and an unknown time-varying disturbance
$\vd: \mathbb{R}^n \times \mathbb{R}^m \times \mathbb{Z}_{\geq 0} \rightarrow \mathbb{R}^n$,
% $F:\mathbb{R}^n \times \mathbb{R}^m \times \mathbb{Z}_{\geq 0} \rightarrow \mathbb{R}^n$ 
with a bounded reward function $R:\mathbb{R}^n \times \mathbb{R}^m \rightarrow [0,1]$.
Given an initial condition $\vx_0$, the decision-making problem is to maximize the sum of the reward function over an infinite horizon:
\begin{align} \label{eq:decision_making_problem}
    \vx^*_{\infty}, \vu^*_{\infty} &= \argmax_{\substack{\vx_{\infty} \in \mathbb{R}^\infty \\ \vu_{\infty} \in \mathbb{R}^\infty}}
    \sum_{k=1}^\infty \gamma^{k-1} R(\vx_k, \vu_k) \\
    \text{s.t. }&\quad \vx_{k+1} = F_\text{nom}(\vx_k, \vu_{k+1}) + \vd(\vx_k, \vu_{k+1}, k) \nonumber \\
    &\quad \vx_{k+1} \in \statespace, \quad \vu_{k+1} \in \actionspace, \quad \forall k \in \mathbb{Z}_{\geq 0} \nonumber
\end{align}
where $0 \leq \gamma < 1$ is a discount factor and $k$ is the physical time.
The problem~\eqref{eq:decision_making_problem} is a discounted infinite-horizon Markov Decision Process, given by the tuple $\langle \statespace, \actionspace, F, R, D, \gamma \rangle$ with $F:\mathbb{R}^n \times \mathbb{R}^m \rightarrow \mathbb{R}^n$ being the nominal dynamics plus disturbance.
Manipulation problems are challenging because $F$ is a nondifferentiable function that carries information about the contact between bodies. 
For example, these dynamics can be modeled with a linear complementarity problem~\cite{mirtich1996impulse}.
Here $\vx_{\infty} = [\vx_1, \vx_2, \dots]$ denotes an infinite sequence of states, and likewise for $\vu_{\infty}$.

\subsection{Model Predictive Trees (MPT) Method} \label{sec:mpt}
% We propose our algorithm, \ac{mpt}, a standard implementation suitable for well-modeled dynamics where $\vd$ is small, and an adaptive method for dynamics with time-varying disturbances, where the planning process benefits from incorporating estimations of the disturbance.
% Our proposed algorithm consists of a receding horizon \ac{uct}-based planner, which we couple with a contraction-theoretic controller.
Our proposed algorithm, specified in Algorithm~\ref{algo:mpt}, has two components: a receding horizon \ac{uct}-based planner and a contraction-theoretic controller.
% The planner provides a real-time desired trajectory with novel contributions in the reuse of subtree information for computational benefit.
The planner provides a real-time desired trajectory, and the controller provides exponential stability to the desired trajectory. 
% Moreover, contraction-theoretic control guarantees trajectory stability, provides a framework to analyze robustness to disturbances.
Moreover, contraction-theoretic control guarantees robust stability and provides a framework to analyze the limitations of tree reuse.

% We close the loop on the planning process at each iteration by comparing the simulated state to the actual state of the robot, correcting the tree when the state has drifted significantly.
% This loop closure ensures the simulated tree state and physical state do not diverge.

As part of \ac{mpt}, we incorporate a time-varying estimate of the disturbance $\vd$ into the planner.
% We present \ac{mpt}, usable as-is for well-modeled dynamics ($\vd$ small) and an adaptive variant, \ac{ampt}, where the planning process benefits from incorporating time-varying estimates of the disturbance.
% The adaptive method, incorporating estimates of the disturbed dynamics, uses an updated model in the planner to account for disturbances.
By combining with the user's choice of an online estimator (with some possibilities being adaptive control~\cite{slotine1991applied}, basis library regression~\cite{williams2016aggressive, lew2022safe}, or Bayesian filtering~\cite{ragan2023bayesian}), \ac{mpt} uses estimates of the disturbance $\hat{\vd}_k: \mathbb{R}^n \times \mathbb{R}^m \rightarrow \mathbb{R}^n$ to plan over the model:
\begin{equation}
    \vx_{k+1} = \hat{F}_k(\vx_k, \vu_{k+1}) = F_\text{nom}(\vx_k, \vu_{k+1}) + \hat{\vd}_k(\vx_k, \vu_{k+1}).\nonumber
\end{equation}
The design and stability of the controller have important implications for the accuracy of subtree reuse under changing dynamics, a connection analyzed in Sec.~\ref{sec:analysis}.

The planning portion of the \ac{mpt} algorithm (Algorithm \ref{algo:mpt}) performs sequential tree searches in receding horizon fashion, 
building an incrementally lengthening desired trajectory.
We make a distinction between two time indices: physical time is the passing of time in the real-world and is denoted with $k$, and simulation time is the time index used in the internal simulation of the tree search and is denoted with $j$.

At each physical time step $k = 1, 2, \dots$, \ac{mpt} performs a search with the \ac{uct} algorithm.  
The search runs a large number of fixed-depth trajectories $\ell = 1, \dots, L$, approximating the infinite-horizon problem as a $K$-depth problem with a value estimate $\hat{V}$.
The objective of Equation \eqref{eq:decision_making_problem} is modified to $\sum_{k=1}^K \gamma^{k-1} R(\vx_k, \vu_k) + \hat{V}(\vx_K)$, for
$\hat{V}: \mathbb{R}^n \!\rightarrow \mathbb{R}_{\geq 0}$.
% $\hat{V}$

% Each rollout walks forward in simulation time starting at the next physical time $k+1$ and proceeding for $K$ steps, $j=k+1, \dots, K+k+1$.
The collection of rollouts form a decision tree $\mathtt{T}_{k+1}$ that holds information about the cumulative reward (in the manner of Eq.~\eqref{eq:decision_making_problem}) and number of visits to each node of the tree.
The index $k+1$ of $\mathtt{T}_{k+1}$ indicates that the root of the decision tree grown at time $k$ has corresponding time $k+1$.

The exploration-exploitation tradeoff of the \ac{uct} algorithm guarantees convergence to the optimal solution of a decision-making problem.
Exploration encourages the tree search to investigate new areas of the space of trajectories and exploitation encourages a refinement of the search in parts of the space that have yielded high rewards, with sampled trajectories quickly concentrating to high-valued regions of space.
The upper confidence bound formula is as follows, where at each parent node $\text{p}$, the value decides through which child node $\text{c}$ to further refine the search:
\begin{equation}\label{eq:uct}
    \textup{UCT(c)} =
    \frac{\text{c}.V}{\text{c}.N} + \epsilon \sqrt{\frac{\log{(\text{p}.N)}}{\text{c}.N}},
\end{equation}
where ``$.V$'' and ``$.N$'' refer to the cumulative value and number of visits to a node, respectively.
Upon completing $L$ \ac{uct}-guided rollouts, the search returns the best action and resulting state out of the search tree $\mathtt{T}_{k+1}$.

The key novelty of our algorithm is the reuse of the selected subtree from the previous iteration to hotstart the tree search at the current time step.
When an action $\vu_{k+1}$ and corresponding child state $\vx_{k+1}$ are selected as ``best'' out of tree $\mathtt{T}_k$, we trim the search tree $\mathtt{T}_k$ at the connection between $\vx_k$ and $\vx_{k+1}$, keeping the trajectories that start at $\vx_{k+1}$. 
The subtree reuse procedure is shown in Figure~\ref{fig:tree_reuse}.
Analyzed in Sec.~\ref{sec:results}, tree reuse enables a more effective search, through which computational power is spent refining high-quality regions of the space of trajectories, rather than re-searching from scratch.

Running in receding horizon fashion, we budget one time step of computational time to the planner, beginning the next iteration's solve with the predicted state $\vx_{k+1}$ while the controller is following the trajectory from $\vx_k$ to $\vx_{k+1}$.

To ensure the system arrives near state $\vx_{k+1}$ when the tree rooted there is ready, we compose our planner with a contraction-theoretic controller, denoted $C$ in the pseudocode.
This controller provides exponential stability to the desired trajectory produced by the planner.
We analyze the stability of our proposed control method and the size of steady-state error as a function of the disturbance $\vd$ and the disturbance estimates $\hat{\vd}$ in Sec.~\ref{sec:analysis}.

We additionally use a tree reset condition to close the loop on the planning process (Line \ref{lin:re-rooting_trim}).
This condition limits the drift between the simulated tree state and the physical state, resetting the tree at a threshold to ensure the simulated tree state and physical state do not diverge.

\SetAlFnt{\small}
\begin{algorithm}[htpb]
\caption{Model Predictive Trees}
\label{algo:mpt}

\SetKwProg{Def}{def}{:}{}
\nosemic
\nonl\text{Parameters:} $\tau: $ Reset threshold \;
\dosemic
\Def{$\textup{Model\_Predictive\_Trees}(\vx_0; \langle \statespace, \actionspace, F_\textup{nom}, R, D, \gamma \rangle, C)$}{
    \tcp{create an initial tree}
    $\mathtt{T}'_1.\text{root} = (\vx_0, \mathbf{0}, [~], 0, 0)$ \;
    $\hat{F}_1 = F_\text{nom}$ \;
    \For{$k=1, \dots, \infty$}{
        \tcp{build tree}
        $\mathtt{T}_k = \textup{UCT\_search}(\mathtt{T}'_k; \langle \statespace, \actionspace, \hat{F}_k, R, D, \gamma \rangle)$ \;
        \tcp{extract desired trajectory}
        \nosemic $\mathtt{T}_k.\text{root}.\text{best\_child} = \argmax\{\text{c}.V / \text{c}.N\}$ \;
        \dosemic \Indp $\text{for c in } \mathtt{T}_k.\text{root}.\text{children}$ \;
        \Indm $\vx^d_k = \mathtt{T}_k.\text{root}.\vx$ \;
        $\vu^d_k = \mathtt{T}_k.\text{root}.\text{best\_child}.\vu$ \;
        \tcp{follow desired trajectory}
        $\vu_{k+1} = C(\vx_k, \vx^d_k, \vu^d_{k+1}; \hat{F}_k)$ \;
        $\vx_{k+1} = \text{rollout}(\vx_k, \vu_{k+1})$ \label{lin:real_rollout}\;
        % \uIf{\textup{aMPT}}{
            \tcp{estimate disturbance}
            $\hat{\vd}_{k+1} = \text{dynamics\_estimate}(\dots)$ \label{lin:dynamics_estimate} \;
            $\hat{F}_{k+1}(\cdot, \cdot) = F_\text{nom}(\cdot, \cdot) + \hat{\vd}_{k+1}(\cdot, \cdot)$ \;
        % }\uElse{
        %     $\hat{F}_{k+1}(\cdot, \cdot) = \hat{F}_k(\cdot, \cdot)$
        % }
        \tcp{trim tree}
        $\mathtt{T}'_{k+1}.\text{root} = \mathtt{T}_k.\text{root}.\text{best\_child}$ \label{lin:re-rooting_trim} \;
        \uIf{$\|\mathtt{T}'_{k+1}.\textup{root}.\vx - \vx_{k+1}\| > \tau$}{
            $\mathtt{T}'_{k+1}.\text{root} = (\vx_{k+1}, \mathbf{0}, [~], 0, 0)$
        }
    }
}

\setcounter{AlgoLine}{0}
\nosemic
\nonl\text{Search parameters:}\;
\Indp
    \nonl $L:$ Number of iterations,
    $b:$ Branching factor,\;
    \nonl $K:$ Search depth,
    $\epsilon:$ Exploration constant \;
\Indm
\dosemic
\Def{$\textup{UCT\_search}(\mathtt{T}_k; \langle \statespace, \actionspace, F, R, D, \gamma \rangle)$}{
    \For{$\ell=1, \dots, L$}{
        $\text{path} = [\mathtt{T}_k.\text{root}]$ \;
        \tcp{rollout}
        \For{$j=k, \dots, K+k$}{
            \uIf{$\text{len}(\textup{path}[-1].\textup{children}) < b$}{
                $\vu_{j+1} \sim U$ ; \tcp{sample action}
                $\vx_{j+1} = F(\text{path}[-1].\vx, \vu_{j+1})$ ; \tcp{step}
                $\text{next\_node} = (\vx_{j+1}, \vu_{j+1}, [~], 0, 0)$ \;
                $\text{path}[-1].\text{children}.\text{append}(\text{next\_node})$
            } \uElse {
                $\text{next\_node} = \underset{\text{c} \in \text{path}[-1]}{\argmax} \bigg\{\frac{\text{c}.V}{\text{c}.N} + \epsilon \sqrt{\frac{\log{(\text{path}[-1].N)}}{\text{c}.N}}\bigg\}$ \label{lin:upper_confidence_bound}\;
            }
            $\text{path}.\text{append}(\text{next\_node})$
        }
        \tcp{backpropagate}
        $\text{cumulative\_reward} = 0$ \;
        \For{$\textup{node } in \textup{ path}.\textup{reversed}()$}{
            $\text{node}.N \pluseq 1$ \;
            $\text{node}.V \pluseq \text{cumulative\_reward}$ \;
            $\text{cumulative\_reward} = \gamma \cdot \text{cumulative\_reward} + R(\text{node}.\vx, \text{node}.\vu)$
        }
    }
    $\textbf{return}\ \mathtt{T}_k$
}

\setcounter{AlgoLine}{0}
\nosemic
\nonl\text{Controller parameters:}\;
\Indp
    \nonl $Q:$ State cost, 
    $R:$ Input cost \;
\Indm
\dosemic
\Def{$C(\vx_k, \vx^d_k, \vu^d_{k+1}; F)$}{
    $A = \frac{\partial F}{\partial x}\rvert_{\vx^d_k, \vu^d_{k+1}}$, 
    $B = \frac{\partial F}{\partial u}\rvert_{\vx^d_k, \vu^d_{k+1}}$ \;
    $M = DARE(A, B, Q, R)$ \label{lin:dare}\;
    $K = (B^\top M B + R)^{-1} B^\top M A$ \;
    $\vu = \vu^d_{k+1} - K(\vx_k - \vx^d_k)$ \;
    $\mathbf{return} \ \vu$
}
\end{algorithm}

% While in \ac{mpt}, we use the nominal dynamics model $F_\text{nom}$ in the planning process, in the adaptive MPT algorithm we plan over a disturbance estimate $\hat{\vd}_k$.

As the disturbance may be drifting over time, past subtree information will have used an old estimate $\hat{\vd}_{k'}$ that is not up-to-date at the current time $k$.
Re-integrating all trajectories in the tree with the up-to-date dynamics information would require a pass over every single node in the tree, removing the intended benefit of reusing the tree.

Our analysis shows that for slowly-changing dynamics, the steady-state tracking error introduced by the use of past estimates is bounded.
Furthermore, understanding the connection between dynamics error, steady-state tracking error, and the contraction metric used to stabilize the system allows us to perform informed hyperparameter tuning.
We analyze this connection in Sec.~\ref{sec:analysis}.
\subsubsection*{Pseudocode Notes}
In Algorithm \ref{algo:mpt}, 
% we consider a tree $\mathtt{T}$ as a collection of nodes originating at a root node $\mathtt{T}.\text{root}$, where 
a node is a tuple 
$(\vx,~\vu,~\text{children},~V,~N)$ consisting of a state, the action that led to the state, its list of child nodes, the cumulative value in the subtree below this node, and the total number of visits to this node, respectively.
Line~\ref{lin:real_rollout} is taken to be the real-world application of $\vu_{k+1}$.
$DARE(A, B, Q, R)$ is the Discrete Algebraic Riccati Equation \eqref{eq:DARE}.
In Line~\ref{lin:dynamics_estimate}, we put a placeholder for the user's choice of dynamics estimator, which may be a function of time $k$, the state/input $(\vx_k, \vu_k)$, desired state/input $(\vx^d_k, \vu^d_k)$, or other latent variables.

\section{Theoretical Results}\label{sec:analysis}

In this section, 
% we derive the connection between the tree search parameters and steady-state tracking error. 
% To achieve this, 
we review discrete contraction analysis, derive the robust stability of our controller, and relate the tree reuse to the steady-state error of the proposed controller.
Whereas previous work has used contraction theory in planning to stabilize local trajectories to an existing global plan, both in the  optimization~\cite{singh2017robust} and learning~\cite{tsukamoto2021learning} context, we use contraction theory to analyze the tree reset procedure. 

\subsubsection*{Notations}
All norms are the standard 2-norm unless otherwise specified.
$\mathtt{I}_p$ denotes the $p \times p$ identity matrix.
A square symmetric matrix $A$ is positive definite ($A \succ 0$), positive semidefinite ($A \succeq 0$), negative definite ($A \prec 0$), or negative semidefinite ($A \preceq 0$) if its eigenvalues are positive, nonnegative, negative, or nonpositive, respectively.
$\lambda_\text{max}(A)$ and $\lambda_\text{min}(A)$ are the largest and smallest eigenvalues of $A$.

\subsection{Discrete-Time Contraction Theory}

% We start with a review of contraction theory as it applies to discrete-time systems.
Consider a discrete-time, time-varying dynamical system 
\begin{equation}
    \vq_{k+1} = F(\vq_k, k), \label{eq:undisturbed_system}
\end{equation}
for time $k \in \mathbb{Z}_{\geq 0}$, state $\vq: \mathbb{Z}_{\geq 0} \rightarrow \mathbb{R}^n$, and a bounded transition function $F: \mathbb{R}^n \times \mathbb{Z}_{\geq 0} \rightarrow \mathbb{R}^n$.
% The boundedness of $\Phi$ guarantees a solution exists for all time~\cite{me} for all initial states $\vq_0$.

We consider the infinitesimal variation of our system: $
    \delta \vq_{k+1} = \frac{\partial F}{\partial \vq}(\vq_k, k) \delta \vq_k.$ 
   % \label{eq:differential_undisturbed_system}
% which can be understood as the differential change in state $\vq_{k+1}$ induced by an infinitesimal variation in $\vq_k$.
% Consider two solution trajectories $\vx_k$ and $\vz_k$ of \eqref{eq:undisturbed_system}, both members of a smooth parameterized family of solutions $\vq_k(\mu)$ for parameter $\mu \in [0,1]$ satisfying $\vq_k(\mu=0) = \vx_k$ and $\vq_k(\mu=1) = \vz_k$.
% Here the variation with respect to $\mu$ is $\delta \vq_k = \frac{\partial \vq_k}{\partial \mu}$, with differential dynamics given by \eqref{eq:differential_undisturbed_system}.
% We analyze the distance between the trajectories by their path length integral, where for each time $k$, the distance is $\int_{\vx_k}^{\vz_k} \| \delta \vq_k \|$.
% By the Fundamental Theorem of Calculus, integrating $\delta \vq_k$ along $\mu=[0,1]$ yields, for all $k$,
% \begin{equation}
%     \vz_k - \vx_k = \int_0^1 \delta \vq_k d\mu.\label{eq:path_integral_trajectory}
% \end{equation}
We say the system~(\ref{eq:undisturbed_system}) is \textit{contracting} if all solutions exponentially converge to a single trajectory.
\begin{theorem}\label{thm:contraction}
    A necessary and sufficient condition for (\ref{eq:undisturbed_system}) to be contracting~\cite{tsukamoto2021contraction} is the existence of a uniformly positive definite matrix $M(\vq, k) = \Theta(\vq, k)^\top \Theta(\vq, k) \in \mathbb{R}^{n \times n}$, called a contraction metric, where $\Theta$ defines a smooth and invertible coordinate transformation $\delta \vp = \Theta(\vq, k) \delta \vq$, which $\forall \vq, k$:
    \begin{equation}
        \frac{\partial F}{\partial \vq}(\vq, k)^\top M(\vq, k) \frac{\partial F}{\partial \vq}(\vq, k) - \alpha^2 M(\vq, k) \preceq 0,
        \label{eq:discrete_contraction_condition}
    \end{equation}
    for constant $0 \leq \alpha < 1$, called the \textit{contraction rate}.
\end{theorem}

Contraction analysis also extends to discrete-time systems with disturbances.
Consider now a perturbed system:
\begin{equation}
    \vq_{k+1} = F(\vq_k, k) + \sigma(k)
    \label{eq:disturbed_system}.
\end{equation}
% which has variational dynamics
% \begin{equation}
%     \delta \vq_k = \frac{\partial \Phi}{\partial \vq}(\vq_{k-1}, k) \delta \vq_{k-1}
% \end{equation}

\begin{theorem}\label{thm:disturbed_contraction}
    Let the system (\ref{eq:undisturbed_system}) be a contracting system with metric $M = \Theta^\top \Theta$, contraction rate $0 \leq \alpha < 1$, and a particular solution $\vx_k$.
    If $\mathcal{L}^\infty (\Theta \sigma) < \infty$, then any solution $\vz_k$ to the perturbed system (\ref{eq:disturbed_system}) converges exponentially to an error ball around $\vx_k$.
    Furthermore, if we assume $M$ is uniformly bounded as $\underline{m}\mathtt{I}_n \preceq M \preceq \overline{m}\mathtt{I}_n$ and $\mathcal{L}^\infty (\sigma) \leq \overline{\sigma}$, the solutions satisfy, for all $k$:
    \begin{equation}
        \| \vz_k - \vx_k \| \leq \alpha^k\sqrt{\frac{\overline{m}}{\underline{m}}} \| \vz_0 - \vx_0 \| 
        + \frac{\overline{\sigma}}{1 - \alpha}\sqrt{\frac{\overline{m}}{\underline{m}}} (1 - \alpha^k).
        \nonumber
    \end{equation}
\end{theorem}

\begin{proof}
    The proof is shown in~\cite{tsukamoto2021contraction}. \qedhere
\end{proof}

\subsection{Exponential Convergence to a Desired Trajectory}
With the tools of discrete-time contraction, we proceed to analyze the tracking performance of our proposed algorithm.
So far in the analysis we have considered dynamical systems without control inputs;
we now present a constructive proof for a feedback law that guarantees the contraction of a discrete-time control system:
\begin{equation}
    \vx_{k+1} = F(\vx_k, \vu_{k+1}) \label{eq:undisturbed_control_system}.
\end{equation}
The proposed controller is a locally-linearized Ricatti controller that, under suitable assumptions, stabilizes the system to a desired trajectory $(\vx^d_k, \vu^d_k)$ for $k \in \mathbb{Z}_{\geq 0}$.
% We make some simplifying assumptions that guarantee the ability to construct a feedback controller that stabilizes the system (\ref{eq:undisturbed_control_system}) about a desired trajectory $(\vx^d_k, \vu^d_k)$ for $k \in \mathbb{Z}_{\geq 0}$.
% 
\begin{assumption} \label{assumption:stabilizable_observable}
    For positive definite cost matrices $Q$ and $R$, we assume the linearized system at each $k$ given by 
    \begin{equation}
        A(\vx^d_k, \vu^d_{k+1}) = \nabla_\vx F \rvert_{(\vx^d_k, \vu^d_{k+1})},\ B(\vx^d_k, \vu^d_{k+1}) = \nabla_\vu F \rvert_{(\vx^d_k, \vu^d_{k+1})}, \nonumber
    \end{equation}
    satisfies $(A, B)$ are stabilizable and $(A, Q^\frac{1}{2})$ are observable.
    This guarantees a unique positive definite solution exists to the Discrete Algebraic Riccati Equation (DARE)~\cite{mao2008existence}:
    \begin{equation}
        M = A^\top M A - (A^\top M B) (R + B^\top M B)^{-1} (B^\top M A) + Q\label{eq:DARE}.
    \end{equation}
% \end{assumption}
% \begin{assumption}\label{assumption:dare}
    Furthemore, we assume the solution to DARE~\eqref{eq:DARE} is uniformly bounded over the state space as $\underline{m}\mathtt{I}_n \preceq M \preceq \overline{m}\mathtt{I}_n$.
\end{assumption}

\begin{theorem}\label{thm:dare_controller}
    Under Assumption~\ref{assumption:stabilizable_observable}, consider the feedback law 
    \begin{equation}\vu_{k+1} = \vu^d_{k+1} - K(\vx_k - \vx^d_k),
    \label{eq:pd_ctrl}
    \end{equation}
    with $K = (R + B^\top M B)^{-1} (B^\top M A)$ and $M$ the solution to $DARE(A, B, Q, R)$~\eqref{eq:DARE}.
    The feedback law in~\eqref{eq:pd_ctrl} yields a closed-loop system that contracts to $\vx^d_k$ with metric $M$ and rate $1 > \alpha \geq \sqrt{1 - \frac{\lambda_\text{min}(Q)}{\overline{m}}}$.
\end{theorem}

\begin{proof}
    Our goal is to show that, for $A_{cl} = A - BK$, \mbox{$A_{cl}^\top M A_{cl} - \alpha^2 M \preceq 0$} holds $\forall \vx, k$.
    Manipulating,
    \begin{align}
        &A_{cl}^\top M A_{cl} - \alpha^2 M = (A - BK)^\top M (A - BK) - \alpha^2 M \nonumber \\
        &= A^\top M A - A^\top M B K - K^\top B^\top M A + K^\top B^\top M B K - \alpha^2 M. \nonumber
    \end{align}
    Plugging in the definition of $K$, note that $K^\top B^\top M B K = K^\top B^\top M A - K^\top R K$, and consequently, the above becomes
    \begin{align}
        &= A^\top M A - A^\top M B K - K^\top R K - \alpha^2 M \nonumber \\
        &= A^\top M A - A^\top M B (R + B^\top M B)^{-1} (B^\top M A) - K^\top R K - \alpha^2 M. \nonumber
    \end{align}
    As $M$ solves $DARE(A, B, Q, R)$, the above becomes
    \begin{equation}
        = M - Q - K^\top R K - \alpha^2 M = (1 - \alpha^2) M - Q - K^\top R K, \nonumber
    \end{equation}
    where $(1 - \alpha^2) M - Q - K^\top R K \preceq 0$ holds if $(1 - \alpha^2) M - Q \preceq 0$, which holds if $(1 - \alpha^2) \overline{m} - \lambda_\text{min}(Q) \leq 0$.
    Thus, the system is contracting with rate $\alpha \geq \sqrt{1 - \frac{\lambda_\text{min}(Q)}{\overline{m}}}$.
\end{proof}

\begin{remark}
    We note that the linearizations in the controller are made about the desired trajectory $(\vx^d_k, \vu^d_{k+1})$, but  if Assumption~\ref{assumption:stabilizable_observable} holds for the true state and desired input $(\vx_k, \vu^d_{k+1})$, we can linearize there.
\end{remark}

The proposed feedback law also enjoys a straightforward robustness result.
Consider a trajectory $(\vx^d_k, \vu^d_k)$ that is a solution to the system (\ref{eq:undisturbed_control_system}), and
consider the disturbed system:
\begin{equation}
    \vx_{k+1} = F(\vx_k, \vu_{k+1}) + \sigma(k) \label{eq:disturbed_control_system}.
\end{equation}
\begin{lemma}\label{lemma:dare_controller_robustness}
    Under Assumption~\ref{assumption:stabilizable_observable}, the proposed control law renders the closed-loop system (\ref{eq:disturbed_control_system}) exponentially stable to an error ball around the desired trajectory:
    \begin{equation}
        \| \vx_k - \vx^d_k \| \leq \alpha^k \sqrt{\frac{\overline{m}}{\underline{m}}} \| \vx_0 - \vx^d_0 \| + \frac{\overline{\sigma}}{1 - \alpha}\sqrt{\frac{\overline{m}}{\underline{m}}} (1 - \alpha^k), \nonumber
    \end{equation}
    where $\mathcal{L}^\infty(\sigma) \leq \overline{\sigma}$ and $\alpha = \sqrt{1 - \frac{\lambda_\text{min}(Q)}{\overline{m}}}$.
\end{lemma}
\begin{proof}
    The proof follows from the application of Theorem~\ref{thm:disturbed_contraction} to the closed-loop system in Theorem~\ref{thm:dare_controller}.
\end{proof}

% Lemma~\ref{lemma:dare_controller_robustness} is directly applicable to the standard Model Predictive Trees algorithm.
% The desired trajectory is a solution to $\vx_{k+1} = F_\text{nom}(\vx_k, \vu_{k+1})$, with the dynamics disturbance $\vd(\vx, \vu, k)$ in Equation (\ref{eq:decision_making_problem}) exactly the disturbance $\sigma$.

We apply this robustness result by analyzing the error introduced by reusing subtrees as the disturbance changes over time.
Past subtrees will have been built with an old estimate, introducing error in the difference between the past and current dynamics estimates.
We first characterize the effect of a time-varying disturbance.
% For the purpose of analysis, we make several assumptions on the disturbance and disturbance estimation.
\begin{assumption}
\label{assumption:slowly_changing_disturbance}
    Consider the disturbance $\vd$ in Equation (\ref{eq:decision_making_problem}).
    We suppose that $\vd$ is ``slowly changing'': that the temporal difference of $\vd$ is bounded.
    Furthermore, we assume the dynamics estimate available to the algorithm at a time $k$ is $\varepsilon$-accurate. 
    There exists $\eta, \varepsilon \in \mathbb{R}_+$ such that for all $\vx, \vu, k$:
    \begin{align}
        \|\vd(\vx, \vu, k+1) - \vd(\vx, \vu, k)\| &\leq \eta \nonumber \\
        \| \vd(\vx, \vu, k) - \hat{\vd}_k(\vx, \vu) \| &\leq \varepsilon \nonumber,
    \end{align}
\end{assumption}

Under these assumptions, we can quantify the error introduced by reusing incorrect information from the past and understand the effect it has on the tracking error.
\begin{theorem}
    Under Assumptions~\ref{assumption:stabilizable_observable}-\ref{assumption:slowly_changing_disturbance}, the steady-state tracking error in \ac{mpt} is bounded as:
    \begin{equation} \label{eq:changing_dynamics_steady_state_error}
        \| \vx_\infty - \vx^d_\infty \| \leq \sqrt{\frac{\overline{m}}{\underline{m}}}\frac{(K+1) \eta + \varepsilon}{1 - \alpha},
    \end{equation}
    for $K$ the depth of the tree search and $\alpha = \sqrt{1 - \frac{\lambda_\text{min}(Q)}{\overline{m}}}$.
\end{theorem}

\begin{proof}
    At time step $i$, with corresponding dynamics estimate $\hat{\vd}_i$, the trajectories in the search run for simulation time $j=i+1, \dots, i+K+1$.
    As such, the maximal time difference between a dynamics estimate and when a tree branch using that dynamics estimate becomes part of the desired trajectory is $K+1$ time steps.
    % Note that there is one more time step of delay than the search horizon as we can only estimate the disturbance after it has affected the dynamics.
    Therefore, when the physical time is $k=i+K+1$, by Assumption~\ref{assumption:slowly_changing_disturbance},
    \begin{equation}
        \| \vd(\vx, \vu, i+K+1) - \hat{\vd}_i(\vx, \vu) \| \leq (K+1) \eta + \varepsilon. \nonumber
    \end{equation}
    As the desired trajectory at $k=i+K+1$ satisfies the dynamics $\vx_{k+1} = F_\text{nom}(\vx_k, \vu_{k+1}) + \hat{\vd}_i(\vx_k, \vu_{k+1})$ and the actual rollout satisfies $\vx_{k+1} = F_\text{nom}(\vx_k, \vu_{k+1}) + \vd(\vx_k, \vu_{k+1}, k)$,
    the tracking error of following the desired trajectory is:
    \begin{equation}
        \| \vx_k - \vx^d_k \| \leq \alpha^k \sqrt{\frac{\overline{m}}{\underline{m}}} \| \vx_0 - \vx^d_0 \| + \frac{(K+1) \eta + \varepsilon}{1 - \alpha}\sqrt{\frac{\overline{m}}{\underline{m}}} (1 - \alpha^k). \nonumber
    \end{equation}
    Letting $k \rightarrow \infty$ yields a steady-state tracking error:
    \begin{equation*}
        \| \vx_\infty - \vx^d_\infty \| \leq \sqrt{\frac{\overline{m}}{\underline{m}}}\frac{(K+1) \eta + \varepsilon}{1 - \alpha}. \qedhere
    \end{equation*}
\end{proof}

The expression~\eqref{eq:changing_dynamics_steady_state_error} dictates the limit of tree reuse in the presence of changing dynamics. 
With this understanding of how the depth of the tree search affects the steady-state tracking error, we can conduct informed parameter design when planning our tree search.
For a given steady-state error threshold, a tradeoff exists between how quickly the disturbance is changing (given by $\eta$) and how far in the future we can search with tree reuse.

% For a threshold $\Delta$, we can set the horizon of our tree search until $K$ is the largest such that the bound in Equation (\ref{eq:changing_dynamics_steady_state_error}) is less than $\Delta$.

% In the next subsection, we will see the role that the depth of the tree search plays on how much information we save when reusing subtrees.

% \subsection{Amount of Tree Reuse}

% To analyze the effect of tree reuse, we will compare our method, based on the tree traversal policy UCT, to the ``naïve'' Monte Carlo Tree Search, where the child selection happens uniformly at random.
% We will refer to this search strategy as plain Monte Carlo (pMC)~\cite{kocsis2006improved}, where an implementation of pMC would have a random selection in Algorithm~\ref{algo:mpt}, Line~\ref{lin:upper_confidence_bound}.

\section{Experimental Results and Discussion}\label{sec:results}

We demonstrate \ac{mpt} on an autonomous vehicle testbed in simulation and hardware.

\subsection{Experimental Setup}
In our experiments, our algorithm solves a planar nonprehensile manipulation task, with reward given for pushing a cylindrical object (a barrel) to a goal position.
Let the state be $\vx = [x, y, \theta, x_o, y_o]^\top \in \mathbb{R}^5$, where $(x, y)$ is the inertial position of the vehicle in meters, $\theta$ is the heading in radians, and $(x_o, y_o)$ is the position of the center of the barrel in meters.
The control inputs are $\vu = [V, \delta]^\top \in \mathbb{R}^2$, where $V$ is speed in meters per second and $\delta$ is steering angle in radians, describing an Ackermann car.
The nonlinear dynamics are:
\begin{align}
    \begin{bmatrix}
        x_{k+1} \\ y_{k+1} \\ \theta_{k+1}
    \end{bmatrix}
    =
    \begin{bmatrix}
        x_k \\ y_k \\ \theta_k
    \end{bmatrix}
    + \Delta t
    \begin{bmatrix}
        V_{k+1} \cos(\theta_k) \\ V_{k+1} \sin(\theta_k) \\ \frac{V_{k+1}}{l} \tan(\delta_{k+1})
    \end{bmatrix}
    \nonumber \\
    \begin{bmatrix}
        x_{o, k+1} \\ y_{o, k+1}
    \end{bmatrix}
    =
    \textup{LCP}\left(
        \begin{bmatrix}x_{o,k} \\ y_{o,k}\end{bmatrix},
        \begin{bmatrix}x_{k+1} \\ y_{k+1} \\ \theta_{k+1}\end{bmatrix}
    \right) \nonumber
\end{align}
with $\Delta t$ being the time step and $l$ the wheelbase length.
The states $x_o$, $y_o$ are found by numerically solving the non-penetration constraints with respect to the car geometry in Fig.~\ref{fig:car_drawing} as a linear complementarity problem (LCP), as in~\cite{mirtich1996impulse}.

\begin{figure}[htpb]
    \centering
    \includegraphics[width=1.0\linewidth]{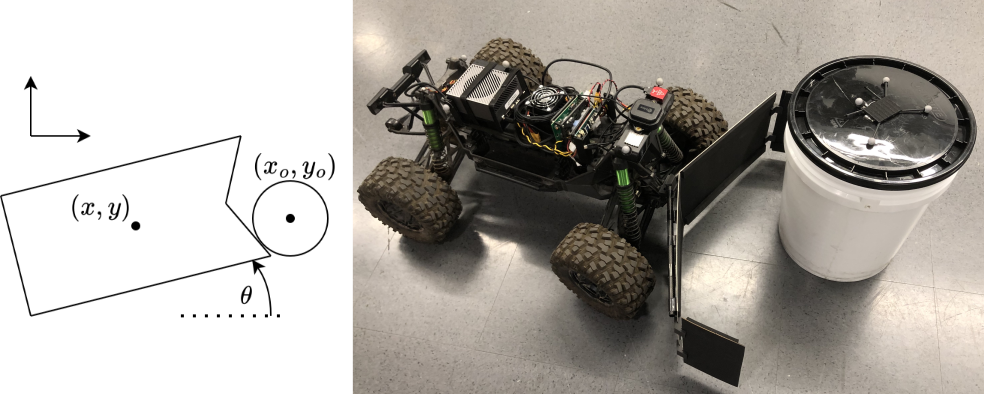}
    \caption{\textbf{Left:} The states and collision geometry of the simulation model used in the experiments. \textbf{Right:} The autonomous vehicle platform and barrel equipped with sensors for state estimation and compute for running our algorithm.}
    \label{fig:car_drawing}
\end{figure}

The reward function is the sum of a nominal reward and a term that increases as the object position approaches $(x_g, y_g)$:
\begin{align}
    R\left(\vx,
    \vu\nonumber\right)
    =
    0.1 + 0.9 \left(1 - \frac{1}{D}\sqrt{(x_o - x_g)^2 + (y_o - y_g)^2}\right)
    \label{eq:sparse_reward}
\end{align}
for normalizing constant $D$.
This reward is sparse because when the vehicle is not in contact with the barrel, no improvement in reward is available until first, contact is made and second, the barrel is pushed toward the goal.

The input limits are $|V| \leq 1~\mathrm{m/s}, \ |\delta| \leq 0.42~\mathrm{rad}$, according to the steering limits of our platform. 
For \ac{mpt} (and the \ac{uct} baseline below), we discretize the action space as $(V, \delta) \in \{ (0, 0), (\pm 1, 0), (\pm 1, \pm 0.42) \}$, sampling uniformly without replacement during tree growth.

\subsection{Baselines}

We compare our method against 3 baselines: (1) \ac{uct} deployed with no tree reuse, referred to as ``\ac{uct}'' in our experiments. This algorithm operates in the same way as \ac{mpt}, but the next root node contains no children from the previous iteration;
(2) a cross-entropy motion planner (\ac{cem}) implemented based on~\cite{kobilarov2012cross}  with a Gaussian input distribution, ten iterations, and 10\% elite particle fraction; and (3) \ac{cem} that hotstarts sampling with the optimal solution of the previous iteration (\ac{cem}-Reuse).

% One perspective on the \ac{uct} algorithm is as a categorical cross-entropy method that uses the upper confidence bound \eqref{eq:uct} to ``sample'' from the distribution at each node.
% A full investigation of this connection is beyond the scope of this study, but continued comparison to quartile thresholding cross-entropy methods~\cite{kobilarov2012cross} and exponentially weighted cross-entropy methods~\cite{williams2016aggressive} is 
% \sjc{worth further examination. what does this mean? Here or in future work?}

\subsection{Numerical Experiments}

% We compare the performance of MPT and the baselines on the manipulation task for a variety of initial states.
% An additional point of comparison is the performance of each algorithm on the \textit{dense} (as opposed to \textit{sparse}) variant of the task, where we have added a designed heuristic to the reward function.
% In the dense experiment, we modify the reward:
% \begin{align}
%     R(&\begin{bmatrix}
%         x\ y\ \theta\ x_o\ y_o
%     \end{bmatrix}^\top,
%     \begin{bmatrix}
%         V\ \delta
%     \end{bmatrix}^\top) = 0.1\nonumber\\
%     &+ 0.8 \left(1 - \frac{1}{D_1}\sqrt{(x_o - x_g)^2 + (y_o - y_g)^2}\right) \nonumber\\
%     &+ 0.1 \left(1 - \frac{1}{D_2}\sqrt{(x - x_o)^2 + (y - y_o)^2}\right)
%     \label{eq:dense_reward}
% \end{align}
% for normalizing constants $D_1, D_2$.
% Here we have added a term that offers a small reward for the vehicle position $(x,y)$ being near the object position $(x_o, y_o)$.

In Fig.~\ref{fig:heatmap_mosaic}, we compare the performance of MPT against the baselines of UCT, CEM, and CEM-Reuse on a grid of initial states.
For $\theta_0 = 0,\ x_{o, 0} = y_{o, 0} = 0$, we vary the initial $x$ and $y$ position of the car over a $4$m $\times$ $4$m space. 
The goal is to push the barrel from its initial position at $(0, 0)$ to $(x_g, y_g) = (4, 0)$.
Here, the value is calculated as the realized (undiscounted) cumulative reward of running each planner in receding horizon fashion for 100 time steps, executing the first proposed action and replanning at each time step.

\begin{figure}[ht]
    \centering
    \includegraphics[width=1.0\linewidth]{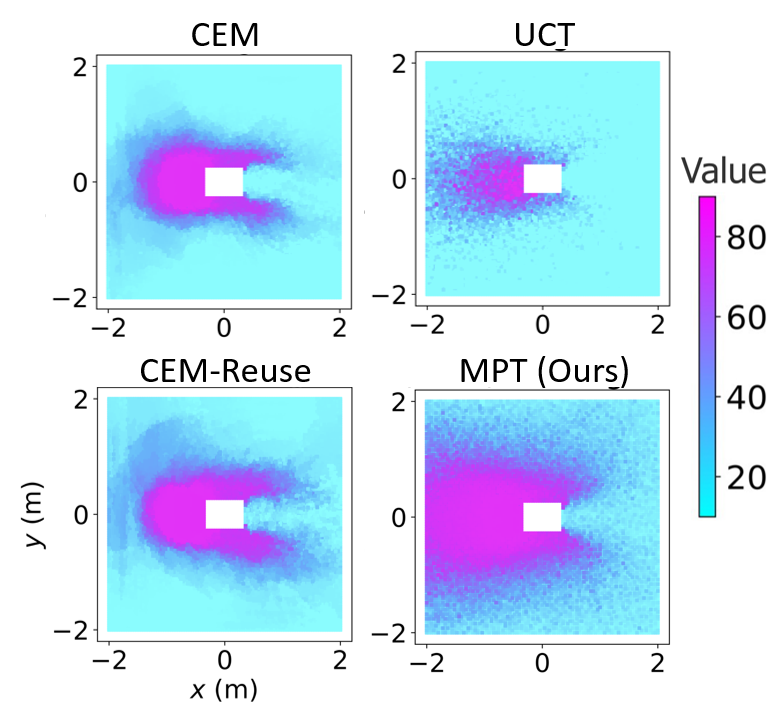}
    \caption{For a grid of $(x,y)$ initial car positions, we color each point according to the accumulated value of running each algorithm (CEM, CEM-Reuse, UCT, MPT).
    Purple indicates higher value.
    These simulations were generated with $L=200$, a planning horizon of $10$, and a simulation depth of $100$. The value shown is averaged across ten runs at each initial condition.
    Our proposed method, \ac{mpt}, provides the best average cumulative reward of all methods, with a significant improvement over the baseline \ac{uct} method.
    }
    \label{fig:heatmap_mosaic}
\end{figure}

The value produced by each method averaged over the state space is summarized below.
Information reuse results in a significant improvement between \ac{uct} and \ac{mpt}.
Whereas \ac{uct} is outperformed by \ac{cem}, the improvement due to reusing information results in \ac{mpt} having a $29.9\%$ higher value than \ac{cem}-Reuse, the next best method.

\begin{center}
\begin{tabular}{|c|c|c|c|c|}
    \hline
    \textbf{Method} & \textbf{\ac{cem}} & \makecell{\textbf{\ac{cem}-} \\ \textbf{Reuse}} & \textbf{\ac{uct}} & \makecell{\textbf{\ac{mpt}} \\ \textbf{(Ours)}} \\
    \hline \hline
    Average Value & 23.63 & 28.97 & 16.79 & 37.64 \\
    \hline
    \makecell{Reuse $\%$ \\ Improvement} & - & $22.6\%$ & - & $124\%$ \\ 
    \hline
\end{tabular}
\end{center}

In the task, \ac{cem} methods are unable to reliably find a solution unless the car is initialized close to the barrel.
Each method performs most consistently when the vehicle starts directly to the left of the barrel, where driving forward will push the barrel to the goal.

\ac{mpt} is able to find high-valued solutions even when the initial position is far from the barrel.
\ac{mpt} can quickly find these ``needle in a haystack'' solutions that require a coordinated maneuver to make contact with then push the barrel to the goal.
The reuse of the search trees of previous iterations allows \ac{mpt} to quickly concentrate its search on high-valued trajectories, without wasting computational effort re-searching through low-valued trajectories.

\subsection{Sample Efficiency}

We examine sample efficiency by considering one initial condition $\begin{bmatrix}x \ y \ \theta \ x_o \ y_o\end{bmatrix}^\top = \begin{bmatrix}\text{-}1.5 \ \text{-}0.5 \ 0 \ 0 \ 0\end{bmatrix}^\top$ and goal position $(x_g, y_g) = (4, 0)$.
We deploy each algorithm in receding horizon fashion where at each time step, $L$ simulations are run and the first step of the plan is taken.
As before, we evaluate the cumulative reward. 
We visualize this metric in Fig.~\ref{fig:v_over_n} vs.~the number of simulations $L$.

\begin{figure}[ht]
    \centering
    \includegraphics[width=0.9\linewidth]{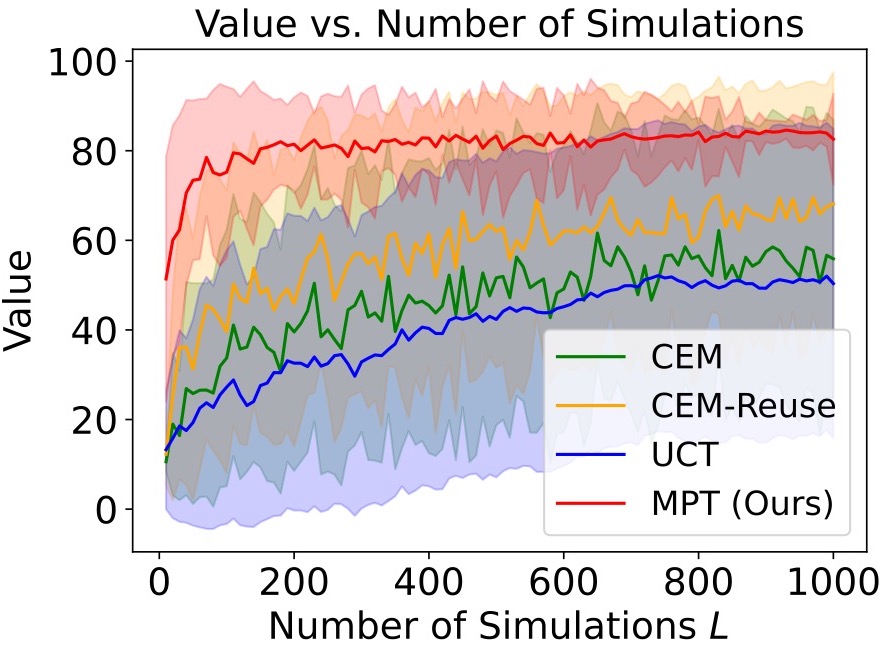}
    \caption{
    % \textbf{Top:} 
    The value of the trajectory produced by each planning method versus the number of simulations. 
    For each $L$, $100$ trials are run, and the average value is plotted with one standard deviation error bar.
    Our proposed algorithm (\ac{mpt}) significantly outperforms the baselines and has a less noisy estimate.
    % \textbf{Bottom:}
    % One example solution found by \ac{mpt} with $L=200$ is plotted.
    % The barrel is moved from its initial position $(0, 0)$ to $(4, 0)$.
    }
    \label{fig:v_over_n}
\end{figure}
Our proposed algorithm greatly outperforms the baselines, with a rapid rise in value, reaching an asymptotic limit at $L=180$. 
The competing baselines exhibit a much slower increase in value, highlighting the sample efficiency of \ac{mpt}.
Furthermore, the value estimates produced by the baselines are significantly noisier than that of \ac{mpt}.
We extend the simulation count to see where the average value produced by each baseline draws level to the asymptotic limit found by \ac{mpt}, with \ac{uct} not catching up in the considered range.
\begin{center}
% \begin{tabular}{|c|c|c|c|}
%     \hline
%     \textbf{Method} & \textbf{\ac{cem}} & \textbf{\ac{cem}-Reuse} & \textbf{\ac{uct}} \\
%     \hline \hline
%     \makecell{$L$ to match \\ \ac{mpt}} & 5200 & 3000 & $>$30000 \\
%     \hline
% \end{tabular}
\begin{tabular}{|c|c|c|c|c|}
    \hline
    \textbf{Method} & 
    \textbf{\ac{mpt}} &
    \textbf{\ac{cem}} & \textbf{\ac{cem}-Reuse} & \textbf{\ac{uct}} \\
    \hline \hline
    \makecell{$L$ needed to \\ reach $V\!=80$} & 180 & 5200 & 3000 & $>$30000 \\
    \hline
\end{tabular}
\end{center}

\subsection{Hardware Results}

\begin{figure}[htpb]
    \centering
    \includegraphics[width=0.9\linewidth]{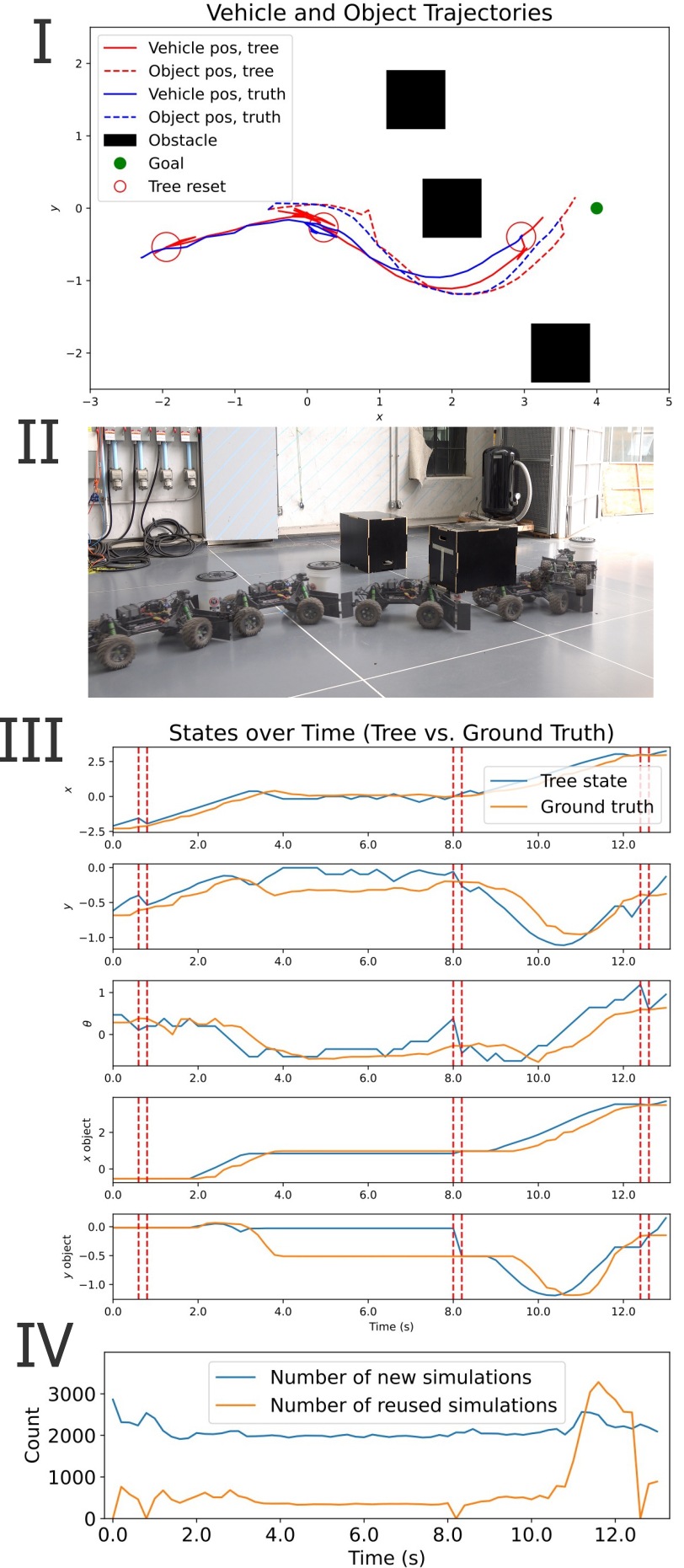}
    \caption{Our algorithm plans for and executes a solution onboard an autonomous vehicle testbed. 
    \textbf{I:} The trajectory of the vehicle (solid line) and the pushed object (dashed line) are shown, with the planned trajectory in red and the actual trajectory in blue.
    The tree resets are circled in red.
    \textbf{II:} An overlay of the trajectory of the vehicle and barrel over the course of the experiment in the Caltech Center for Autonomous Systems and Technologies.
    \textbf{III:} The states as simulated by the tree (blue) and as measured by the motion capture (orange).
    The tree reset instances are each shown as a pair of red lines.
    \textbf{IV:} The number of simulations saved by reusing the tree is shown.
    On average, one third of the new simulations are carried over to the next planning iteration.
    Near the end of the experiment, when the optimal behavior is easy to find, the number of reused simulations rises dramatically.
    The decision trees grown here are highly concentrated to the optimal actions at each depth.
    }
    \label{fig:hardware-run}
\end{figure}
We verify the ability of \ac{mpt} to be deployed on hardware by implementing our algorithm on the autonomous vehicle testbed shown in Fig.~\ref{fig:car_drawing}.
We task \ac{mpt} to solve the barrel-pushing task in an environment with three obstacles. \ac{mpt} is able to plan in real time and execute a $12-$second pushing operation that maneuvers the barrel around the obstacles to the goal position.
The trajectory of the vehicle and the barrel around the obstacles are shown in Fig.~\ref{fig:hardware-run}.

Our algorithm is able to plan through high-level behaviors of making and breaking contact with the barrel.
Halfway through the experiment, the vehicle stops, backs up, and re-positions itself behind the barrel to make a push around an obstacle.
Such a maneuver is only possible if planning through the hybrid dynamics, a distinct advantage of our proposed method.
% Existing state-of-the-art methods are alternatively too inefficient or unable to plan through the dynamics, rendering the observed behavior unique to our algorithm.
We show that state-of-the-art baselines are either too sample-inefficient or unable to plan through the dynamics, rendering the observed behavior unique to \ac{mpt}.

At each planning iteration, we check the tree reset condition with $\tau=0.5$.
At three times in the experiment, the tree is reset when the state of the vehicle or barrel diverge from the simulated state.
The resets are all triggered by a mismatch in the $\theta$ state or the object position, meaning the dynamics mismatch is occurring in the steering model and contact dynamics.
In this task, we use a constant estimated dynamics model, but the real physics of the contact include friction effects, deformation, and other unmodeled dynamics that contribute to the model mismatch, resulting in resets. 

In this experiment, our \ac{mpt} planner is running at $5$ Hz on an onboard NVIDIA Jetson Orin, running approximately 2100 simulations every 0.2 s.
We measure the position of our vehicle and the barrel with motion capture.
For our numerical and hardware experiments, we use a value estimate $\hat{V} \equiv 0$.
If data is available, an option is to train a neural network value estimator, as in related works in tree search~\cite{riviere2021neural}.

\section{Conclusion}

% We presented \ac{mpt}, a new receding horizon tree search that quickly solves sparse planning tasks with few samples.
% We provided theoretical analysis guaranteeing the stability of our method and robustness to model mismatch.
% We demonstrated the performance of our algorithm against state-of-the-art sampling-based planners and showed our algorithm working onboard hardware at 5 Hz.
% Planning directly through the underlying hybrid dynamics, our planner, in real time, was able to solve a nonprehensile manipulation task to push a target barrel through an obstacle field.

% \ac{mpt} proposes a powerful receding horizon planning framework that readdresses information reuse and solver hotstarting by reusing the full subtree associated with the previous optimal solution.
% We expect rethinking how motion planning algorithm address information reuse will provide significant improvement to a wide variety of algorithms deployed in a wide range of applications.

We present \ac{mpt}, a new receding horizon planning framework that reuses a rich set of information from prior solver iterations to solve challenging planning problems.
Our theoretical analysis guarantees the stability of our method and robustness to model mismatch, characterizing the limitations of tree reuse.
We use our planner to produce solutions in real-time for a challenging nonprehensile manipulation task to push a target barrel through an obstacle field.
We demonstrate the performance improvement of our algorithm against state-of-the-art sampling-based planners, isolating the effect of replanning with partial and complete information reuse.
Our results suggest information reuse is an important area of study that can provide significant improvement to a wide variety of algorithms and applications.

% By providing the solver with more information than just the previous optimal solution, \ac{mpt} searches more efficiently than other methods.
% At a high level, MPT proposes the first framework of continuous computation in sampling-based methods; the robot's planning process is no longer occurring at a fixed frequency, and is instead an uninterrupted tree growth, coupled to the real world by the sequence of realized actions. 
% \sjc{Please improve. Why is MPT so important and novel in the field of robot (motion) planning?}
% \todo{Add more contribution.}

% \addtolength{\textheight}{-12cm}
% Shortens last page of document

\section*{ACKNOWLEDGMENT}

The authors would like to thank E.~S.~Lupu, J.~A.~Preiss, and F.~Xie for technical discussions.

%%%%%%%%%%%%%%%%%%%%%%%%%%%%%%%%%%%%%%%%%%%%%%%%%%%%%%%%%%%%%

% \newpage
% \bibliographystyle{ieeeconf}
\bibliographystyle{ieeetr}
\bibliography{papers}

\end{document}

%% file: glossary.tex
\usepackage{acro}

\DeclareAcronym{mcts}{
  short=MCTS,
  long=Monte Carlo Tree Search,
}

\DeclareAcronym{cpu}{
  short=CPU,
  long=Central Processing Unit,
}

\DeclareAcronym{scp}{
  short=SCP,
  long=Sequential Convex Programming,
}

\DeclareAcronym{uct}{
  long=Upper Confidence Bound for Trees,
  short=UCT,
}

\DeclareAcronym{rrt}{
  short=RRT,
  long=Rapidly-exploring Random Trees,
}

\DeclareAcronym{prm}{
  long=Probabilistic Roadmaps,
  short=PRM,
}

\DeclareAcronym{cem}{
  short=CEM,
  long=Cross-entropy Motion Planning,
}

\DeclareAcronym{mpt}{
  long=Model Predictive Trees,
  short=MPT,
}

\DeclareAcronym{mppi}{
  long=Model Predictive Path Integral Control,
  short=MPPI,
}

\DeclareAcronym{mpc}{
  long=Model Predictive Control,
  short=MPC,
}